\documentclass[12pt]{article}
\usepackage{amsmath}
\usepackage{graphicx}
\usepackage{enumerate}
\usepackage{url} % not crucial - just used below for the URL 

\newcommand{\blind}{1}

\usepackage{epsfig}
\usepackage{caption}
\usepackage{graphicx}

\usepackage{booktabs}

\usepackage{flafter} % figures only appear after they are referenced

\usepackage{amsthm}

\newtheorem{theorem}{Theorem}
\newtheorem{lemma}{Lemma}
\newtheorem{corollary}{Corollary}
\newtheorem{definition}{Definition}
\newtheorem{examplex}{Example}
\newenvironment{example}
{\pushQED{\qed}\examplex}
{\popQED\endexamplex}

\begin{document}
	
	\def\spacingset#1{\renewcommand{\baselinestretch}%
		{#1}\small\normalsize} \spacingset{1}

	%%%%%%%%%%%%%%%%%%%%%%%%%%%%%%%%%%%%%%%%%%%%%%%%%%%%%%%%%%%%%%%%%%%%%%%%%%%%%%

	\if1\blind
	{
		\title{\bf An Analysis of Loss Functions for Binary Classification and Regression}
		\author{Jeffrey S Buzas \\ %1\thanks{
				%The authors gratefully acknowledge \textit{please remember to list all relevant funding sources in the unblinded version}}\hspace{.2cm}\\
			Department of Mathematics and Statistics, \\
			University of Vermont, 82 University Place, Burlington, VT 05401-1455
			}
		\maketitle
	} \fi
	
	\if0\blind
	{
		\bigskip
		\bigskip
		\bigskip
		\begin{center}
			{\LARGE\bf An Analysis of Loss Functions for Binary Classification and Regression}
		\end{center}
		\medskip
	} \fi
	
	\bigskip
	\begin{abstract}
	This paper explores connections between margin-based loss functions and consistency in binary classification and regression applications. It is shown that a large class of margin-based loss functions for 
	binary classification/regression result in estimating scores equivalent to log-likelihood scores weighted by an even function.  A simple characterization for conformable (consistent) loss functions is given, which allows for straightforward comparison of different losses, including exponential loss,
	logistic loss, and others.  The characterization is used to construct a new Huber-type loss function for the
	logistic model.   
	A simple relation between the margin and standardized logistic regression residuals
	is derived, demonstrating that all margin-based loss can be viewed as loss functions of squared standardized logistic regression residuals.  The relation provides new, straightforward interpretations for exponential and logistic loss, and aids in understanding why exponential loss is sensitive to outliers. In particular, it is shown that minimizing empirical  exponential loss is equivalent to minimizing the sum of squared standardized logistic regression residuals. The relation also provides new insight into the AdaBoost algorithm. 
	\end{abstract}
	
	\noindent%
	\vfill
	
	\newpage
	\spacingset{1.9} % DON'T change the spacing!

\section{Introduction}

Binary classification is an important topic in statistical and machine learning. Margin-based loss functions form the backbone of non-parametric approaches to binary classification, and have been studied extensively over the last few decades. Excellent introductions are given in Bishop (2006)\nocite{Bishop2006} and Hastie et al. (2009)\nocite{ESL2009}.   In contrast,
model-based likelihood methods are typically used in regression applications for binary outcomes.  This paper provides a straightforward derivation for a simple characterization of the class of conformable (consistent),  margin-based loss functions. The characterization provides a direct method for comparing different
loss functions and constructing new ones.   It is shown that derivatives of loss functions in this class are equivalent to log-likelihood scores weighted by an even function, thereby establishing a strong connection between classification using margin-based loss and likelihood estimation.  

A simple algebraic relation is derived that establishes an equivalence between the margin and standardized logistic regression residuals (SLRRs).  The relation implies 
that all margin-based loss functions can be considered as loss functions of squared SLRRs.  This viewpoint provides interesting new perspectives on commonly used loss functions, including exponential loss, which underlies the AdaBoost algorithm.  It is shown that minimizing empirical exponential loss is equivalent to minimizing the sum of squared SLRRs. It is argued that AdaBoost can be interpreted as forward stage-wise regression where the objective function to be minimized is the sum of squared SLRR's, weighted by squared SLRR's from previous stages.  The relation between SLRRs and margins does not appear have been previously known.

An interesting approach for constructing loss functions based on ideas from the probability elicitation literature was given in Masnadi-Shirazi and Vasconcelos (2008) \nocite{SV2008} (see also Buja et al. (2006) \nocite{Buja2006}). 
In Masnadi-Shirazi and Vasconcelos (2015)\nocite{SV2015}, the same authors extended their construction to provide a characterization of loss functions consistent for what they term the generalized logit-link, with a focus on developing loss functions with strong regularization properties.   Their derivation used the fact the generalized logit-links are invertible.  The characterization of loss functions here
has a similar form to that given for the generalized logit-link in Masnadi-Shirazi and Vasconcelos (2015)\nocite{SV2015}. However, the characterization here
 is not restricted to generalized logit-links, nor does it require an invertible probability model. The derivation herein is straightforward, and, importantly, the characterization allows for direct comparison of loss functions using different paramaterizations.  Informative parameterizations 
of the characterization are explored, including in terms of the log-likelihood score.

\section{Notation and review of margin-based classification}

The following establishes notation, provides a review of some of the key elements in binary classification, and defines  the notion of conformable loss functions. 
Let $x\in \Re^p$ denote a feature vector and $y^*\in\{-1,1\}$ a binary class indicator. The joint probability distribution for $(x,y^*)$,
denoted $P$, is  not specified.  Interest is in finding a classification
rule $C(x)\in\{-1,1\}$ that minimizes classification error. Define the 0-1 loss function as 
$\phi_{\operatorname{0-1}}(y^*C(x))=I(y^*C(x)<0)=I(C(x)\ne y^*)$ where $I(\cdot)$ is the indicator function. The 0-1 loss function 
rewards/penalizes correct/incorrect classification.   The conditional  risk for 0-1 loss and a 
classifier $C(x)$ is then 
\begin{equation*}
	\begin{split}
R_\mathsf{C,\operatorname{0-1}} & =E[\phi_{\operatorname{0-1}}(y^*C(x))\mid x] \\
& =E[I({C}(x)\ne y^*)\mid x] = \phi_{\operatorname{0-1}}(C(x))p(x)+\phi_{\operatorname{0-1}}(-C(x))(1-p(x)) \\
&=I(C(x)\ne 1)p(x)+I(C(x)\ne -1)(1-p(x))=\mbox{Pr}(y^*\ne C(x)\mid x)  
\end{split}
\end{equation*}
where $p(x)\equiv\mbox{Pr}(y^*=1\mid x)$. 
The latter expression is the probability of classification error using the classifier $C(x)$.  It is well known that the  value of $C(x)$ minimizing the risk from
0-1 loss is the so-called Bayes
decision rule given by $C^*(x)=\mbox{sign}[2p(x)-1]$.  When $p(x)=1/2$, the classifier is indeterminant.  Here it is assumed this occurs with probability
zero and so can be ignored, thereby avoiding technicalities that can obscure the presentation while offering little insight.

The conditional distribution $\mbox{Pr}(y^*=1\mid x)$ is unknown as there are no distributional assumptions on $P$.  Therefore, classification rules are estimated non-parametrically from data using loss functions.  Classification functions estimated
from data  that converge to the optimal classifier $C^*(x)$ are called Bayes consistent.  This definition of consistency is analogous to that of Fisher consistency in parameter estimation, see \nocite{Lin2004} Lin (2004).

 Typically $C(x)$ takes the form $C(x)=\mbox{sign}[f(x)]$ where $f:\Re^p\rightarrow\Re$ is a function mapping the features $x$ to the real numbers.  The task is then to estimate $f(x)$.  

The margin $v$ is defined as the product $v=y^*f(x)$.  Note that $C(x)=\mbox{sign}[f(x)]$  correctly classifies an observation if and only if $v>0$.
Estimation of $f(x)$  using $0-1$ loss presents a difficult optimization problem because the loss is not smooth or convex (or quasi-convex), and the resulting classifiers can have poor finite sample properties, see Vapnik (1998) \nocite{vapnik1998}.   Therefore, continuous loss functions are used as surrogates for 0-1 loss.

  The available data are $n$ independent pairs $(x_1,y^*_1),\dots,(x_n,y^*_n)$ from the distribution $P$.  Estimation of $f(x)$ is typically accomplished by minimizing the empirical risk of a smooth margin-based loss function $\phi(v)$ by computing
\[
\hat f(x) =\arg\min_{f\in\mathcal{F}}\frac{1}{n}\sum_{i=1}^n\phi(v_i)
\]
where $\mathcal{F}$ is a large class of functions.

\begin{example}
Exponential loss, given by $\phi(v)=e^{-v/2}$, is commonly employed and is the basis of the AdaBoost algorithm.  Let $\mathcal{F}$ be the class of linear models, that is  
 $\mathcal{F}\equiv\{f(x): f(x)=x^T\beta,  \beta\in\Re^p\}$.  
The classifier estimated from independent pairs $(x_1,y^*_1),\dots,(x_n,y^*_n)$ is  $\hat C(x)=\mbox{sign}[\hat f(x)]$
where $\hat f(x) = x^T\hat\beta$, and
$\hat\beta$ minimizes the linear model/exponential loss empirical risk:
\[
\hat \beta =\arg\min_{\beta}\frac{1}{n}\sum_{i=1}^n e^{-y^*_i x_i^T\beta/2}.
\]

\end{example}

Linear models represent a rather restrictive class of functions for $\mathcal{F}$. A class of functions with enormous flexibility and widely used in applications is given by
\begin{equation}\label{classFunctions}
\mathcal{F}\equiv\left\{ f(x;\beta): f(x;\beta) = \sum_{m=1}^M \theta_m b(x; \gamma_m) \right\}
\end{equation}
where $\beta=\{\theta_m,\gamma_m\}_{m=1}^M$ are unknown parameters and $b(x:\gamma_m)$ are basis functions, see  Hastie et al. (2009, page 341) \nocite{ESL2009}.  

Population level properties of loss functions are important for understanding whether the resultant classification rules are Bayes consistent.    To study properties at the population level, let $\mathcal{F}$ represent the class of all measurable functions, and note the conditional risk function for margin-based loss is given by 
\[
R_\mathsf{C}(f(x)) = E[\phi(Y^*f(x))\mid x] = \phi(f(x))p(x)+\phi(-f(x))(1-p(x)).
\]   
 The optimal value for $f(x)$ at the population level is defined as
\[
f^*(x)=\arg\min_{f\in\mathcal{F}} R_\mathsf{C}(f(x)).
\]
It is often possible to optimize $R_\mathsf{C}(f(x))$ through differentiation. For a differentiable loss function $\phi(v)$, the identity $\frac{\partial}{\partial f} R_C(f)=0$ yields 
\begin{equation}\label{derivRisk}
\frac{\phi^\prime(-f^*(x))}{\phi^\prime(f^*(x))}=\frac{p(x)}{1-p(x)}
\end{equation}
where $\phi^\prime(v)=\frac{d}{dv} \phi(v)$.
Let $0\le G(v) \le 1$ be a continuous cumulative distribution function (CDF) such that $G(0)=1/2$.  If the derivative of the loss function  satisfies
%symmetric around $1/2$, i.e. $G(x) = 1-G(-x)$. 
%DO I NEED THE SYMMETRY REQUIREMENT HERE? I DON'T THINK SO. MAYBE ONLY NEED G(0)=1/2.  
\begin{equation}\label{definingeq}
	\frac{\phi^\prime(-v)}{\phi^\prime(v)} = \frac{G(v)}{1-G(v)} 
\end{equation}
then $f^*(x)$ is such that $G(f^*(x))=p(x)\equiv\mbox{Pr}(y^*=1\mid x)$, i.e. at the population level the conditional probability $\mbox{Pr}(y^*=1\mid x)$, which was unspecified, is determined by $G(v)$ and $f^*(x)$.  The preceding analysis
assumes that there does not exist $B\subset \Re^p$ where $\mbox{Pr}(x\in B)>0$ and such that $p(x)=1$ for $x\in B$.  In particular, the case where the features provide perfect prediction is excluded, i.e. when $p(x)=I(x\in B)$ where
$I(\cdot)$ is the indicator function.
  
\begin{definition}\label{ConsistentDef} {\bf Conformable loss functions}:  Let $\phi(v)$ be a differentiable, margin-based loss function and $G(x)$ a continuous CDF where $G(0)=1/2$.  If $\phi^\prime (v)$ satisfies \eqref{definingeq}, we say $\phi(v)$ is {\bf conformable }to $G(x)$.
\end{definition}

\bigskip

The definition is useful because it is easily seen that conformable loss functions yield classification rules that are Bayes consistent, i.e.  converging to the optimal rule $C^*(x)$.  To see this, note that  for a conformable loss $\phi(v)$, it follows that, at the population level,  $C(x)=\mbox{sign}[f^*(x)]=\mbox{sign}[2\mbox{Pr}(y^*=1\mid x)-1]=C^*(x)$.  Then
under standard regularity conditions, $\mbox{sign}[\hat f(x)]\rightarrow\mbox{sign}[f^*(x)]$, i.e. $\hat C(x)\rightarrow C(x) = C^*(x)$.  % This argument requires that $f^*(x)$ satisfy \eqref{derivRisk}.   
An analogous result is established in in Section \ref{regressPerspective} below, namely that conformable loss functions produce consistent estimators in regression parameter estimation contexts.

Conformable loss functions also yield an approximate soft classifier. Soft classifiers directly model 
 $\mbox{Pr}(y^*=1\mid x)$ and then define the classifier as 
$C(x)=\mbox{sign}[2\mbox{Pr}(y^*=1\mid x)-1]$.  A loss function conformable for $G(x)$ provides 
an approximate soft classifier via $\mbox{Pr}(y^*=1\mid x)\approx G(\hat f(x))$.   

While conformable loss functions are Bayes consistent, not all loss functions that result in Bayes consistent classifiers need be conformable in the 
sense of Definition \ref{ConsistentDef}.   Several of the most widely used loss functions are conformable. Examples include exponential, logistic, and squared loss.  Hinge loss, used with support vector machines, is not conformable as it is not differentiable.  However, hinge loss is Bayes consistent.  Weak conditions under which a loss function is Bayes consistent are given in Lin (2004)\nocite{Lin2004}.

The definition of conformable loss functions is related to what  Shen (2005)\nocite{Shen2005} and Buja et al. (2006) \nocite{Buja2006} term `proper' loss functions.  Conformable loss functions are proper. 
%though an explicit definition of proper {\it margin-based} loss functions in the literature is elusive. 
The definition of conformable loss functions focuses on $G(v)$ and requires $\phi(\cdot)$ to be differentiable, but  does not require $G(\cdot)$ to be invertible, and the definition is  independent of $\mathcal{F}$.

In the next section, a characterization of the class of conformable loss functions is given.

\section{A characterization of the class of differentiable loss functions conformable for $G(x)$ }

Let $0\le G(x) \le 1$ be a continuous cumulative distribution function for a distribution that is symmetric around $1/2$, i.e. $G(x) = 1-G(-x)$.  This strengthens the requirement that $G(0)=1/2$.  The following derives a characterization of the class of non-decreasing, differentiable loss functions 
that are conformable for symmetric $G(x)$.   It is straightforward to generate new loss functions
using the characterization, and the characterization is particularly well-suited to understanding and comparing properties of loss functions that are conformable for the logistic distribution.    

   Motivated by \eqref{definingeq}, we seek to characterize the class of functions $\mathcal{C}_q$ defined by 
  $h(w)\in\mathcal{C}_q$ if and only if
\[
h(w)=q(-w) h(-w)
\]
 where $q(w)$ is a function with the property $q(w)q(-w)=1$.  Interest will be in $q(w)=G(w)/(1-G(w))$, but the following result is stated more generally.

\begin{lemma}\label{classlemma} $h(w)\in \mathcal{C}_q$ if and only if 
	\[
	h(w) = q(-w)^{1/2} g(w)
	\]
	where $g(w)$ is an even function.
\end{lemma}
\begin{proof}
	Assume $h(w)\in \mathcal{C}_q$, i.e. that $h(w)=q(-w) h(-w)$. Then we can write $h(w) = q(-w)^{1/2} g(w)$ where
	$g(w)= q(w)^{1/2} h(w)$ and it is easily seen that $g(w)$ so defined is an even function.
	Conversely, if $h(w) = q(-w)^{1/2} g(w)$ where
	$g(w)$ is an even function, then 
	\[
	h(-w) = q(w)^{1/2}  g(-w) = q(w)^{1/2} g(w)=\frac{q(w)^{1/2}}{q(-w)^{1/2}}h(w)
	= q(w)h(w).
	\]
\end{proof}
% Remark:  Note that Lemma \ref{classlemma} holds for any function $q(w)$ satisfying $q(w)q(-w)=1$. 

Let $\mathcal{\tilde C}_q$ represent the class of functions $\mathcal{C}_q$ with the restriction that $g(w)$ is continuous,
 $g(w)\ge 0,$ and where 
 \begin{equation}\label{oddseq}
 q(w)=\frac{G(w)}{1-G(w)}.
 \end{equation}
The restriction $g(w)\ge 0$ is not mathematically necessary, but will result in loss functions that are non-increasing, a property which seems to make the most sense practically.

If $G(w)$ is interpreted as the conditional probability $\mbox{Pr}(y^*=1\mid w)$, $q(\cdot)$ represents the odds of an outcome conditional on $w$.  Useful properties to note are that $q(w)q(-w)=1$,  $\lim_{w\rightarrow\infty} q(w) = \infty$ and $\lim_{w\rightarrow -\infty} q(w) = 0$.

\begin{definition}\label{LG} 
 Denote by $\mathcal{L}_G$ a class of loss functions of the form % $\phi(v)\in \mathcal{L}_G$ if 

\begin{equation}\label{lossfunction}
\phi(v) = k-\int_{0}^{v} h(w) dw =\begin{cases}
	k-\int_0^v h(w)dw, & \text{if $v>0$} \\
	k+\int_v^0 h(w)dw, & \text{if $v\le 0$}
\end{cases}
\end{equation}
where $k>0$ is an arbitrary constant and $h(w)\in\mathcal{\tilde C}_q$.  
\end{definition}
Each $\phi(v)\in\mathcal{L}_G$ is indexed by a weight function $g(w)$ that can be chosen independently from $G(w)$. The derivative of the loss is $\phi^\prime (v)=-h(v)=-q(-v)^{1/2}g(v)$, which is the product of the  square root of the reciprocal of the odds ($q(w)$) and the weight function $g(w)$.

\begin{theorem}\label{lossfuncthr} A margin-based loss function $\phi(v)$ is in the class $\mathcal{L}_G$ if and only if $\phi(v)$ is 
non-increasing,  differentiable, and is conformable to $G(x)$.

\end{theorem}
\begin{proof} Suppose $\phi(v)\in \mathcal{L}_G$, i.e. $\phi(v)$ is given by \eqref{lossfunction}.  Then $\phi(v)$ is differentiable by the Fundemental Theorem of Calculus, and is non-increasing because $\phi^\prime(w)=-h(w)\le 0$. 
The derivative of $\phi(v)$ satisfies  \eqref{definingeq}, and therefore $\phi(v)$ is conformable for $G(x)$.  
Conversely, if $\phi(v)$ is conformable, non-increasing, and differentiable, then $\phi^\prime(v)\in\mathcal{\tilde C}_q$ and  by Lemma \ref{classlemma}, $\phi^\prime(v)=-q(-v)^{1/2}g(v)$ where $q(v)=G(v)/(1-G(v))$ and $g(v)\ge 0$ is an even function. It follows that $\phi(v)$ has the form given in  \eqref{lossfunction}, i.e. $\phi(v)\in \mathcal{L}_G$.

\end{proof}  
 
Constructing loss functions in $\mathcal{L}_G$ is straightforward.  Simply choose a target distribution function
$G(w)$ satisfying $G(w)=1-G(-w)$, construct $q(w)$ and choose a positive even function $g(w)$.  The loss 
function is then essentially the anti-derivative of $h(w) = q(w)^{-1/2}g(w)$.   Some choices of 
$G$ and $g$ will result in intractable integrals. For purposes of minimizing the empirical risk to obtain an
estimate $\hat f(x)$, evaluating the integral isn't necessary if the optimization is done by gradient descent, rendering the inability to evaluate the anti-derivative moot. 

For the logit-link, the characterization given in Theorem \ref{lossfuncthr} is similar to the characterization for the generalized logit-link given in Masnadi-Shirazi and Vasconcelos (2015)\nocite{SV2015} (see their Theorems 5 and 6).   However, the derivation and proof of Theorem \ref{lossfuncthr} is  applicable to all symmetric CDFs ($G(x)$), and the derivation does not require
$G(x)$ to be invertible.

The following two examples are intended to show how loss functions in $\mathcal{L}_G$ can be constructed using  \eqref{lossfunction}.  The intention is not to construct  loss functions for a specific purpose, and therefore justification for the choices of $G$ and $g$ are not considered in these examples. 

\begin{example}
Set $G(w)=(w+1)/2$ for $w\in [-1,1]$ (uniform distribution over $[-1,1]$).  Then for $-1\le w\le 1$,
\[
q(w)^{-1/2}=\sqrt{\frac{1-w}{1+w} }.
\]

Let  $g(w)=4\sqrt{G(w)(1-G(w))}=2\sqrt{(1+w)(1-w)}$, and $k=1$. Squared error loss results:
\[
\phi(v)=1-\int_0^v q(w)^{-1/2}g(w)dw = (1-v)^2.
\]
More generally, for the uniform distribution function it is easily seen that the use of $g(w)=((1-w)(1+w))^{(m+1)/2}$ where $m$ an odd integer results in 
a polynomial loss function of degree $m+1$.   The loss functions in this example all have the population level property that
$\mbox{Pr}(y^*=1\mid x) = G(f^*(x))=(f^*(x)+1)/2$.
\end{example}

\begin{example}\label{logisticexample}
	Suppose interest is in constructing a loss function conformable to the logistic CDF.  Set $G(w)=F(w)\equiv (1+e^{-w})^{-1}$ for $-\infty<w<\infty$.  Then 
	\[
	q(w)^{-1/2}=e^{-w/2}.
	\]
	
	Choose  $g(w)=1/2$ and $k=1$. Exponential loss results:
	\[
	\phi(v)=1-\int_0^v q(w)^{-1/2}g(w)dw = e^{-v/2}.
	\]
	
\end{example}

Several additional examples of constructing loss functions in $\mathcal{L}_F$ are given in Section \ref{LogistSection}, where  $\mathcal{L}_F$ represents the class of non-increasing, differentiable loss functions conformable to the logistic distribution function.

The choice of both $G(w)$ and $g(w)$ can depend on the objectives of the analysis.  In fact, while a conformable loss function may be desirable,  conformability to
a specific choice of $G(w)$ may not be important.  For example, properties of hinge loss are often compared to those of  logistic likelihood loss, despite the fact they each are interpreted differently at the population level. 
 
Likelihood loss functions are well known to have optimal properties when estimation of regression parameters is the goal. Therefore, the weight function for likelihood  loss (also termed cross-entropy loss) may provide
a touchstone for comparison to other weight functions in the class $\mathcal{L}_G$. The weight function
for log-likelihood loss is derived in the following example.  

\bigskip

\begin{example}\label{mleexample}   Here we work in reverse, i.e. we start with the log-likelihood loss function and derive the weight function.  Therefore, we begin by assuming a model for the conditional distribution $y^*\mid x$ so that  a  likelihood can be constructed. For a differentiable distribution function $G(x)$ where $G(x)=1-G(-x)$, suppose that 	$\mbox{Pr}(y^*=1\mid x) = G(f(x))$ for some function $f(x)$.  
Then the log-likelihood loss function, denoted $\phi_L(v)$, can be written in terms of the margin $v=y^*f(x)$:
\[
\phi_L(v) =   \log G(y^*f(x)) = \log G(v). 
\]
Setting $\phi_L^\prime(v)\mid_{v=w}=-q(w)^{-1/2}g(w)$ gives
\[
\frac{d}{dv} \phi_L(v)\mid_{v=w} = -\frac{G^\prime(w)}{G(w)}=-q(w)^{-1/2}g(w). 
\]
Solving for $g(w)$ yields the log-likelihood  weight function
\[
g_L(w)=q(w)^{1/2}\frac{G^\prime(w)}{G(w)}=\frac{G^\prime(w)}{\sqrt{G(w)(1-G(w))}}.
\]
\end{example}

\subsection{Convexity}

Properties of convex loss functions have been studied extensively (see for example \nocite{Bartlett2006} Bartlett et al. (2006)),  though the importance of convex loss has been debated, especially with repect to outlier robustness (see for example \nocite{Zhao2009} Zhao et al. 2009).
Loss functions in $\mathcal{L}_G$ differ only with respect to the choice of $g(w)$, and therefore properties of $g(w)$, relative to $q(w)$, determine whether the loss is convex.   More specifically, convexity of  $\phi(v)\in\mathcal{L}_G$ is  determined by the maximum rate at which $\log [g(w)]$ changes relative to the change in the log-odds;  $\log[q(w)]$.   This is easily seen when $g(w)$ is strictly positive and differentiable.

\begin{lemma}\label{convex}
	Suppose  $\phi(v)\in\mathcal{L}_G$ where $G(w)$ and $g(w)$ are differentiable, and $g(w)>0$.  Then $\phi(v)$ is convex   if and only if $\frac{d}{dv}\log g(v)\le \frac{1}{2}\frac{d}{dv}\log q(v)$.
	%(i.e. $\log[g(w)]$ is Lipschitz continuous with bound $1/2$).
\end{lemma} 
		\begin{proof} 
			From the assumptions, $\phi(v)$ is a twice differentiable, strictly decreasing function.  It is not difficult to show
			that $\phi(v)$ is convex if and only if $\frac{d}{dv}\log [-\phi^\prime(v)]\le 0$.  For $\phi(v)\in\mathcal{L}_G$,
			we have $\phi^\prime(v)=-q(v)^{-1/2}g(v)$. Under the assumption that 
			$G(v)$ and $g(v)$ are differentiable, it follows $\phi(v)$ is convex if and only if 
			\[
			\frac{d}{dv}\log g(v)\le \frac{1}{2}\frac{d}{dv}\log q(v).
			\]
		\end{proof}
		
%		Convexity of the class of loss functions $\mathcal{L}_F$ where $F$ is the logistic distribution function 
%		are discussed further below. 

\subsection{Other parameterizations of $\mathcal{L}_G$}
 
 The derivatives of loss functions in $\mathcal{L}_G$ are the product of a function of $G(v)$ ( the inverse square root of the odds), and an arbitrary non-negative, even weight function.  
 We say the loss function is ``parameterized" in terms of the odds.  
The parameterization of the loss functions in $\mathcal{L}_G$ defined in  \eqref{lossfunction} is not unique. 
%The characterization in \eqref{lossfunction}  is based on weighting the odds, expressed in terms of $G(v)$, with an even function $g(w)$.      
The class of loss functions in $\mathcal{L}_G$ can be defined and interpreted through an infinite number of other equivalent  parameterizations as explicated in the following Corollary to Lemma \ref{lossfuncthr}.   Each parameterization represents the derivative of $\phi(v)\in\mathcal{L}_G$ as the product of a function of $G(v)$ and an arbitrary non-negative, even weight function. 

\begin{corollary}\label{corollaryParam}
	%Suppose that $b(\cdot)$ is a function such that $b(w)=b(1-w)$. 
	%Then $b(G(w))$ is an even function.  
	Suppose the function $b:[0,1]\rightarrow\Re^+$ satisfies $b(x)=b(1-x)$.  Define
	$r(w)=q(w)^{-1/2}b(G(w))$ and
	\[
	h^*(w)=r(w)g(w)
	\]
	where $G(w)$ is a symmetric CDF,  and $g(w)\ge 0$ is an  even function.  Then $h^*(w)\in \mathcal{\tilde C}_q$.

\end{corollary}
\begin{proof}
It is not difficult to show that for $G(w)$ a symmetric CDF, $b^*(w)=b(G(w))$ is an even function if and only if $b(x)=b(1-x)$.  Then $g^*(w)=g(w)b^*(w)$ is an even function and the Corollary follows immediately from Lemma \ref{classlemma}. 
% from Theorem \ref{lossfuncthr}.
\end{proof}
Note that $r(w)$ is a function of $G(w)$ only. The implication of the Corollary is that loss functions in $\mathcal{L}_G$ can be defined as
\begin{equation}\label{param}
	\phi(v) = k-\int_{0}^{v} h^*(w) dw =\begin{cases}
		k-\int_0^v h^*(w)dw, & \text{if $v>0$} \\
		k+\int_v^0 h^*(w)dw, & \text{if $v\le 0$}
	\end{cases}
\end{equation}
where $k>0$ is an arbitrary constant.  Loss functions in $\mathcal{L}_G$ are then interpreted as $r(w)$
 smoothed (integrated) using an even weight function $g(w)$, in which case we say that $\mathcal{L}_G$ is {\it parameterized} in terms of $r(w)=q(w)^{-1/2}b(G(w))$.   Two parameterizations are explored in the following example and theorem.

\begin{example}\label{DFexample}  Suppose that $b(x)=\sqrt{x(1-x)}$ and $\tilde g(w)=1$.  Then 
\[
r(w)=q(w)^{-1/2}b(G(w))= 1-G(w). 
\]
 The class of functions $\mathcal{L}_G$ can be generated via  \eqref{param} 
using $h^*(w)=(1-G(w))g(w)$ where $g(w)\ge 0$ is an even function.  Therefore, loss functions in $\mathcal{L}_G$ can be interpreted
as $1-G(w)$ smoothed via an even function $g(w)$.  Assuming $G(w)$ is differentiable, a natural weight function for this parameterization is to 
choose $g(w)=G^\prime(w)$, i.e.  use the density function corresponding to $G(w)$ as the weight function. Note the density will be an even function because $G(w)$ is assumed symmetric (odd) about $1/2$.  Squared error loss on the CDF scale results, seen as follows.  Set $k=1/8$. Then  $\phi(v) =  1/8-\int_0^v (1-G(w))g(w)dw = 1/8-\int_0^v (1-G(w))G^\prime(w)dw = .5(1-G(v))^2$.   
\end{example}

Assuming that $G(v)$ is differentiable, the following theorem states that derivatives of margin-based loss functions in $\mathcal{L}_G$ are simply weighted versions of the derivative of the log-likelihood.
 
\begin{theorem}\label{LLexample}  
	Suppose $\phi(v)\in\mathcal{L}_G$ and that $G(v)$ is differentiable.   Then 
\[
\phi^\prime(v) = - g(v)\frac{d}{dv}\log [G(v)]
\]	
where $ g(v)\ge 0$ is an even function.  If 	the probability model $\mbox{Pr}(y^*=1\mid x) = G(f(x))$ is assumed,
then $\frac{d}{dv}\log [G(v)]$ represents the derivative of the log-likelihood.

\end{theorem}
\begin{proof}
 Let $b(v) = 1/\sqrt{G(v)(1-G(v))}$ and $\tilde g(v)=  g(v)\frac{d}{dv} G(v)$
	where $ g(v)\ge 0$ is an even function.  Note that as defined $\tilde g(v)$ is an even, non-negative function.   Let
	\[
 h^*(v) =		q(v)^{-1/2}b(G(v))\tilde g(v)  =  g(v) \frac{\frac{d}{dv} G(v)}{G(v)} = g(v)\frac{d}{dv}\log [G(v)] . 
	\] 	
  From Corollary \ref{corollaryParam}, $h^*(v)\in \mathcal{\tilde C}_q$.
   Then the class of functions $\mathcal{L}_G$ can be generated via \eqref{param} 
	with $h^*(w)=g(w)\frac{d}{dw}\log [G(w)]$.  
	From Example \ref{mleexample} and under the assumption that $\mbox{Pr}(y^*=1\mid x) = G(f(x))$, $\frac{d}{dv}\log [G(v)]$  is the derivative of the log-likelihood.
\end{proof}

The utility of a paramaterization  depends on the form of $G(x)$ and/or the application. 
For some distribution functions $G(x)$, the parameterization defined in Example \ref{DFexample}  may
provide advantages in interpretation versus the odds parameterization.  However, the logistic distribution is the canonical distribution
in binary data analysis, and the odds parameterization is easiest to work with in that case.   It is not difficult to show that the parameterizations in Example   \ref{DFexample} and Theorem \ref{LLexample} are equivalent when $G(x)$ is the logistic distribution function.

\subsection{Regression parameter estimating equation perspective}\label{regressPerspective}

The population property of loss functions  $\phi(v)\in\mathcal{L}_G$, namely that $f^*(x)$ satisfies $\mbox{Pr}(y^*=1\mid x)=G(f^*(x))$, 
%$f^*(x)$ can be interpreted on the $G(x)$ scale, 
can also be understood from a regression perspective.  
In fact, the class could have been derived by seeking to characterize the class of margin-based unbiased estimating equations for regression parameter estimation. 

Suppose we postulate a regression model (or soft classifier) 
$\mbox{Pr}(y=1\mid x) = G(f(x;\beta))$ where $G(x)$ is a symmetric distribution function  and
$f(x;\beta)$ is in the class given by  \eqref{classFunctions}.
Then the loss functions in $\mathcal{L}_G$ result in unbiased estimating equations for $\beta$.  Additional standard regularity conditions would ensure consistency of the regression parameter estimates.

\begin{corollary}\label{regressresult}
 Suppose the true population  model is	$\mbox{Pr}(y^*=1\mid x) = G(f(x;\beta))$ where $G(v)=1-G(-v)$, with features (independent variables) 
 $x\in\Re^p$, and $\beta\in\Re^p$ are  regression parameters.  If $\phi(v)\in \mathcal{L}_G$ then
\[
E \Big [ \frac{\partial}{\partial\beta}\phi(y^*f(x;\beta))\mid x\Big ]=0.
\] 
The converse is true if we add the restriction that $\phi(v)$ is non-increasing.
%NOTE: Can probably strengthen this to something like:  Let phi(v) be a margin-based loss function phi results in an unbiased score
%iff phi is in L. Or maybe just deal with margin-based estimating equations. Has to be margin based though! quasi-likelihood scores wouldn't be of the form I am looking at, for example. 
\end{corollary}
\begin{proof}
 Suppose that $\phi(v)\in\mathcal{L}_G$ so that $\phi^\prime(v)=-q(-v)^{1/2}g(v)$. The conditional expectation of the estimating score $\frac{\partial}{\partial\beta}\phi(y^*x^T\beta)$ under the true population model is
	\begin{equation*}
		\begin{split}
		E \Big [ \frac{\partial}{\partial\beta}\phi(y^*f(x;\beta))\mid x\Big ] & = \Bigg ( \phi^\prime(f(x;\beta))G(f(x;\beta)) \\
		& \qquad-\phi^\prime(-f(x;\beta))(1-G(f(x;\beta))) \Bigg )\frac{\partial f(x;\beta)}{\partial\beta} \\
			& = -\Bigg ( q(-f(x;\beta))^{1/2}g(f(x;\beta))G(f(x;\beta)) \\
			&\qquad\qquad -q(f(x;\beta))^{1/2}g(-f(x;\beta))(1-G(f(x;\beta)))
			\Bigg )\frac{\partial f(x;\beta)}{\partial\beta} \\
			& = - \Bigg ( \left (\frac{1-G(f(x;\beta))}{G(f(x;\beta))}\right )^{1/2}G(f(x;\beta))  \\
			&\qquad	- \Big (\frac{G(f(x;\beta))}{1-G(f(x;\beta))}\Big )^{1/2}(1-G(f(x;\beta))) \Bigg )\frac{\partial f(x;\beta)}{\partial\beta} g(f(x;\beta)) \\
			&=0.
		\end{split}
	\end{equation*}
	The penultimate equality follows from the definition of $q(\cdot)$ and because $g(\cdot)$ is an even function.
	
	Conversely, if $E \big [ \frac{\partial}{\partial\beta}\phi(y^*f(x;\beta))\mid x\big ]=0$ then from the first identity above,
	\[
		\frac{\phi^\prime(-f(x;\beta))}{\phi^\prime(f(x;\beta))} = \frac{G(f(x;\beta))}{1-G(f(x;\beta))}.
	\]
 Then $\phi^\prime\in \mathcal{\tilde C}_q$ and by assumption $\phi(v)$ is non-increasing.
From Lemma \ref{classlemma}, it follows that  $\phi^\prime (v) = q(-v)^{1/2} g(v)$ for an even function $g(v)$, and therefore $\phi(v)\in\mathcal{L}_G$. 
\end{proof}

In a regression parameter estimation context, the class $\mathcal{L}_G$ results in estimating equations that are the 
derivative of the log-likelihood weighted by an even, non-negative function. This result follows readily from Theorem \ref{LLexample}  and Corollary \ref{regressresult}. It is stated 
as a Theorem because of the importance of connecting margin-based loss functions to likelihoods.
 \begin{theorem}\label{regressresult2}
 	Suppose the true population  model is	$\mbox{Pr}(y^*=1\mid x) = G(f(x;\beta))$ where $G(v)=1-G(-v)$ and $G(v)$ is differentiable.  Then $\phi(v)\in \mathcal{L}_G$ if and only if
 	\begin{equation}\label{unbiasedeq}
 	 \frac{\partial}{\partial\beta}\phi(y^*f(x;\beta))=-g(f(x;\beta)) \frac{\partial}{\partial\beta}\log [G(y^*f(x;\beta))]
 	\end{equation}
 	where $g(w)\ge 0$ is an even function.
 	%NOTE: Can probably strengthen this to something like:  Let phi(v) be a margin-based loss function phi results in an unbiased score
 	%iff phi is in L. Or maybe just deal with margin-based estimating equations. Has to be margin based though! quasi-likelihood scores wouldn't be of the form I am looking at, for example. 
 \end{theorem}

From \eqref{unbiasedeq}, the  proof of Corollary \ref{regressresult} 
is perhaps more evident, at least to statisticians, as it is well-known that the expectation of the derivative of the log-likelihood is zero under the assumed model.  

Note that a larger class of unbiased estimating equations is obtained
by not restricting $g(v)$ in \eqref{unbiasedeq} to be even (or positive).  However, when $g(v)$ is not an even function, the resulting estimating function will not
be a function of the margin, and thus falls outside of the class of loss functions considered herein. Also, for $G(v)$ symmetric, it is sensible for $g(v)$ to be an even function.

The variance and sensitivity of estimating scores are important for understanding asymptotic properties of the resulting parameter estimates, see \nocite{Godambe1991} Godambe (1991).  The following lemma is stated without proof, as the proof is similar to the proof in the corollary above. 
The  lemma says that the conditional variance of the estimating score depends only on the square of the weight function, but not on $G(v)$.  The second part of the lemma provides an expression for the sensitivity of a loss function in $\mathcal{L}_G$.  The lemma assumes loss functions
 parameterized in terms of the class $\mathcal{\tilde C}_q$ where recall $q(v)^{-1/2} = [(1-G(v))/G(v)]^{1/2}$. 
\begin{lemma} 
	 Suppose the true population  model is	$\mbox{Pr}(y^*=1\mid x) = G(f(x;\beta))$ where $G(v)=1-G(-v)$. The conditional variance of a loss function $\phi(v)\in\mathcal{L}_G$ is given by
	\[
	E \left [ \frac{\partial}{\partial\beta}\phi(y^*f(x;\beta)) \Big (\frac{\partial}{\partial\beta}\phi(y^*f(x;\beta))\Big )^T
	\mid x\right ] 
	= g^2(f(x;\beta))
	\left (\frac{\partial f(x;\beta)}{\partial\beta}  \right ) \left (\frac{\partial f(x;\beta)}{\partial\beta}  \right )^T.
	\]

The sensitivity of a loss function $\phi(v)\in\mathcal{L}_G$ is given by
\[
E \left [ \frac{\partial^2}{\partial\beta\partial\beta^T}\phi(y^*f(x;\beta)) 
\mid x\right ] 
= \left [ \frac{G^\prime(f(x;\beta))g(f(x;\beta)) }{\sqrt{G(f(x;\beta))(1-G(f(x;\beta)))}} \right ]
\left (\frac{\partial f(x;\beta)}{\partial\beta}  \right ) \left (\frac{\partial f(x;\beta)}{\partial\beta}  \right )^T.
\]
\end{lemma}
Note that the variance and sensitivity have equal magnitudes when the derivative of the loss is equal to the likelihood score.  That is, if $g(w)$ is set to the likelihood weight
function derived in Example \ref{mleexample} , the magnitudes of the variance and sensitivity given in the lemmas are equivalent.  This equivalence is well known in statistical maximum likelihood estimation theory, and implies the loss function possess Godambe efficiency (\nocite{Small2003} Small and Wang, 2003).

\bigskip

\section{Equivalence of the Margin and Standardized Logistic Regression Residuals}\label{Sec:marginresid}

As noted above, loss functions conformable to a CDF  provide a bridge between hard and soft classifiers.  Here we show that all margin-based classifiers are linked to the logistic CDF, implying that all margin-based classifiers can be considered as approximately soft classifiers. 

\subsection{The margin and logistic regression residuals}

We establish an exact algebraic relation between the margin and 
standardized logistic regression residuals (SLRRs).  The relation provides 
new insight binary classification using margin-based loss. 

Let $y=(y^*+1)/2\in\{0,1\}$.  In logistic regression settings, the standardized residual, denoted $S(f(x))$, is defined as the 
ratio of the residual to the standard deviation of $y$:

\[
S(f(x))=\frac{y-F(f(x))}{\sqrt{F(f(x))\{1-F(f(x))\}}}.
\]

Hastie et al. (2009) note that the margin plays a role similar to the residuals in regression.  Here we show the margin is exactly equal to a function of 
standardized residuals from the logistic regression model.

\begin{theorem} 
	\label{key}
	The squared standardized logistic regression residual and the margin have the following relation:
	\begin{equation}\label{keyidentity}
		-\log[S^2(f(x))] = y^*f(x).
	\end{equation}

\end{theorem}

\begin{proof}
	Using straightforward algebra, it is not difficult to show that
	\begin{equation}\label{stdresid}
		S(f(x) )=y\exp \left [-\frac{1}{2}f(x)\right ]-(1-y)\exp \left [\frac{1}{2}f(x) \right ]=y^*\exp\left [-\frac{1}{2}y^*f(x)\right ].
	\end{equation}
	Squaring both sides and log transforming proves the theorem. 
\end{proof}

Equation \eqref{keyidentity} implies that $S^2$ is a monotonically decreasing function of the margin, and that $S^2<1$ if and only if $y^*f(x)>0$.  Large positive margins result in small residuals whereas large negative margins result in large residuals.
Note that \eqref{stdresid} implies that $S(1,f(x))=-S(0,-f(x))= -1/S(0,f(x))$ where $S(y,f(x))$ is the 
standardized residual.    If the sign of
$f(x)$ and the class of $y$ are both switched, the sign of the residual changes, but the magnitude of the residual is exactly the same.  

Margin-based loss is often championed as a non-parametric approach to binary classification.  However, 
Theorem \ref{LLexample} showed that loss functions in $\mathcal{L}_G$  are weighted versions of the log- likelihood score,
implying there  is an underlying parametric assumption for these loss functions.
The relation between margins and SLRRs given in Theorem \ref{key} suggests that there is a parametric connection 
for all margin-based loss, even those outside the class $\mathcal{L}_G$.    
The theorem also implies that margin based loss can be considered as a distance-based loss.

Note that the proof of Theorem \ref{key} does not assume or imply that $p(x)=\mbox{Pr}(y=1\mid f(x)) = F(f(x))$, i.e. the result does not require that the population conditional distribution of $y$ given $x$ is given by the logistic model. The theorem also does not imply that every margin-based loss function $\phi(v)$ is conformable to $F(\cdot)$. 
The result does have implications for how margin based classifiers can be interpreted, as discussed below.  

The  result supplies an easily interpretable  metric for the confidence
of a classification.  The relation between the margin and standardized logistic regression residuals implies $F(\hat f(x))$ is
a natural measure of the confidence of the classifier
$\hat C(x)=\mbox{sign}[\hat f(x)]$ for a given value of $x$.   Note that if the population minimizer for a given loss function 
is such that  $p(x)=G(f^*(x))\ne F(f^*(x))$, the confidence measure $F(\hat f(x))$ should not be interpreted as the conditional probability of $y^*=1$ for a given value of $x$. 

The result also provides a connection between soft and hard margin classifiers.  Soft classifiers estimate the conditional probability 
$\mbox{Pr}(y=1\mid f(x))$ and assign $C(x)=\mbox{sign}\{2\mbox{Pr}(y=1\mid f(x))-1\}$, whereas hard classifiers provide a decision boundary but do not necessarily estimate the conditional probability (\nocite{Yufeng2011} Yufeng et al. 2011).  
For any margin based estimation scheme that does not restrict values of the margin, we can approximate a soft classifier via $\mbox{Pr}(y=1\mid f(x))\approx F(\hat f(x))$. 
Of course it is possible to approximate a soft classifier using $\hat f(x)$ and any distribution function symmetric around zero.
The relation between the margin and logistic regression residuals supplies  motivation and justification for using the logistic distribution function for the approximation. 

There is a tension between loss functions in $\mathcal{L}_G$ where $G(x)\ne F(x)$ and Theorem \ref{key}. These loss functions yield $f^*(x)$ satisfying $p(x)=G(f^*(x))$.  On the other hand, Theorem \ref{key} 
implies margin-based loss is a distance-based loss for the logistic model, and therefore the approximation $p(x)\approx F(\hat f(x))$ seems reasonable.

\subsection{Partitioning logistic regression residuals}

The results in this subsection are primarily of interest in logistic regression settings, but they will also be used to provide 
a new perspective on the AdaBoost algorithm. The following theorem shows that standardized logistic regression residuals can be partitioned on a multiplicative scale.  Equivalently, the partition is additive on the logarithmic scale. 
The proof of the following theorem is an immediate consequence of \eqref{stdresid}. 

\begin{theorem}  Suppose that $f(x;\beta)$ is given as in \eqref{classFunctions}, i.e. $f(x;\beta)= \sum_{m=1}^M \theta_m b(x; \gamma_m)$ where $\beta=\{\theta_m,\gamma_m\}_{m=1}^M$.  Then
	\[
	S(f(x;\beta)) = (y^*)^{M+1} \prod_{m=1}^M S(\theta_m b(x;\gamma_m)).
	\]
	
	Then also
	\[
	S^2(f(x;\beta)) =  \prod_{m=1}^M S^2(\theta_m b(x;\gamma_m)).
	\]
	
\end{theorem}
  
The theorem does not say that we can fit a logistic regression model by fitting $M$ individual models, each with one independent variable.  It says the standardized residuals from a model can be partitioned into individual components, where the estimated coefficients in the components are from the fit of the full model.  The full model could be fit by maximizing the likelihood or some other approach. 

The contribution of the $k$th element of $f(x;\beta)$ to the $i$th residual can be interpreted through the squared standardized residuals on the log-scale.  Note that
\[
\frac{\log S^2(\theta_k b(x_i;\gamma_k))}{\log S^2(f(x_i;\beta))}=\frac{\log S^2(\theta_k b(x_i;\gamma_k))}{\sum_{m=1}^M \log S^2(\theta_m b(x_i;\gamma_m))} = \frac{\theta_k b(x_i;\gamma_k)}{\sum_{m=1}^M\theta_m b(x_i;\gamma_m)}.
\]

The contribution of a component can also be interpreted by comparing a component to the geometric mean of all components on a log scale:  

\begin{equation*}
	\begin{split}
		\log \Bigg \{\frac{ S^2(\theta_k b(x_i;\gamma_k))}{ ( S^2(f(x_i;\beta)))^{1/M}}\Bigg \} & =\log \Bigg \{\frac{ S^2(\theta_k b(x_i;\gamma_k))}{\prod_{m=1}^M ( S^2(\theta_m b(x_i;\gamma_m)))^{1/M}} \Bigg\} \\
		& = y^*_i\Big (\theta_k b(x_i;\gamma_k)-\frac{1}{M}\sum_{m=1}^M\theta_m b(x_i;\gamma_m) \Big ).
	\end{split}
\end{equation*}

Potential uses of the partition are histograms of the individual components across observations to identify outliers, or for each observation, histograms of the $M$ components to understand the influence of the predictors on the residual.  
% Potential uses will be explored in future research.  

\section{Loss functions conformable for the logistic distribution}\label{LogistSection}

In this section we consider loss functions interpreted on the logit scale at the population level, i.e.  we consider $G(x)=F(x)=(1+e^{-x})^{-1}$. In the theory of generalized linear regression models, it is well-known that the logistic distribution is the (inverse) canonical link for binary data.    The logistic distribution could also be considered the canonical distribution for construction of loss functions in binary margin-based classification for a couple of reasons.   First, the odds  $q(w)$ arise naturally in the solution to minimizing the conditional  risk, and the log-odds is the inverse of the logistic distribution: $F^{-1}(w)=\log[F(w)/(1-F(w))]=\log q(w)$. Second, as shown in Section \ref{Sec:marginresid} above, the margin is a function of the standardized logistic regression residual. Properties of the (derivative of) conditional risk and logistic distribution translate into analytical tractability and ease of interpretation. For these reasons, the logistic distribution deserves detailed exploration. 

From Example \ref{logisticexample}, 
\[
q(w)^{-1/2}=\Bigg (\frac{1-F(w)}{F(w)}\Bigg )^{1/2} = e^{-w/2}
\]
and the class of loss functions $\phi(v)\in\mathcal{L}_F$ can be written as
\[
\phi(v) = k-\int_{0}^{v} e^{-w/2}g(w) dw,
\]
showing that loss functions in $\mathcal{L}_F$ can be considered as smoothed, weighted versions of the exponential loss function $e^{-w/2}$.   This observation aids in interpretation and 
comparison of loss functions in $\mathcal{L}_F$.

Several of the most commonly used loss functions for binary classification and regression are in $\mathcal{L}_F$, including exponential loss and logisitc  loss (also known as log-likelihood loss, cross entropy loss, and deviance loss). In addition to their use in regression contexts, these two loss functions are the basis of the
popular boosting methods AdaBoost and LogitBoost. Table \ref{logistTable} shows several common loss functions 
in $\mathcal{L}_F$ and
the weight functions $g(w)$ used to generate them. Two new loss functions, Gaussian and Laplace loss, are also included.  These latter two loss functions are explored further below. 

\subsection{Convexity for logistic conformable loss functions}

For the loss functions in Table \ref{logistTable}, exponential and logistic loss are convex, whereas Savage, Gaussian and Laplace loss are not.  These loss functions differ only with respect to the choice of $g(w)$, and therefore properties of $g(w)$ determine the convexity of the loss function.  From Lemma \ref{convex}, it follows that convexity of  $\phi(v)\in\mathcal{L}_F$ occurs if and only if  $\frac{d}{dw}\log g(w)\le 1/2$.   In other words, convexity is  determined by the maximum rate at which $\log [g(w)]$ changes over the real line.  Note this convexity result assumes the parameterization in terms of $q(w)$. 

If the inequality is satisfied, we say $\log[g(w)]$ is Lipschitz continuous with bound $1/2$. 
If  $\log [g(w)]$ is Lipschitz continuous with bound $M/2$ where $1/2<M/2<\infty$, then $\tilde g(w)= g(w)^{1/M}$ is Lipschitz continuous with bound $1/2$ and $\tilde g(w)$ will yield  a convex loss function.  For example,
from Table \ref{logistTable}, Savage loss (see Masnadi-Shirazi and Vasconcelos (2008)\nocite{SV2008}) is obtained when $g(w)= (F(w)(1-F(w)))^{3/2}$.  It follows
that $\frac{d}{dw}\log g(w) = -3/2 +3(1-F(w))<3/2$ where the upper bound ($M/2=3/2$) is sharp. 
Therefore Savage loss is non-convex, but we can modify the weight function to obtain a convex loss. 
With $M=3$, let $\tilde g(w)= g(w)^{1/M}=\Big ( (F(w)(1-F(w)))^{3/2} \Big )^{1/3}=(F(w)(1-F(w)))^{1/2}$.
Note that $\tilde g(w)$ is Lipschitz continuous with bound $1/2$.  The resulting loss function is the logistic likelihood loss (well-known to be convex).  See Figure \ref{logistFigure} for a graphical comparison of 
weight functions.

  \bigskip

  \begin{tabular}{lrrr}
  	\toprule
  	\textbf{Name} & $\phi(v)$ & $g(w)$ & $k$ \\ 
  	\midrule
  	Exponential  & $e^{-\frac{1}{2}v}$ & 1/2 & $1$ \\
  	Logistic  & $\log(1+e^{-v})$ & $\big ( F(w)(1-F(w))\big )^{1/2}$ & $2$ \\
  	Savage  & $(1+e^{v})^{-2}$ & $\big ( F(w)(1-F(w))\big )^{3/2}$ & $2$ \\
  	Gaussian & $1-\Phi(\frac{v+m/2}{\sqrt{m}})$ &  $\frac{1}{\sqrt{2\pi m}}e^{-\frac{1}{2m}w^2-m/8}$ & $1-\Phi(\frac{\sqrt{m}}{2})$ \\
  	Laplace &  $ \begin{cases}
  		e^{-v(1+m)/2}, & \text{if $v>0$} \\
  		1+\frac{m+1}{m-1}(1-e^{v(m-1)/2}), & \text{if $v\le 0$}
  	\end{cases}$ & $\frac{m+1}{2}e^{-m|w|}$ & $1$ \\
  	\bottomrule
  \end{tabular}
{\captionof{table}{Loss functions for the logistic model and the indicated generating weight function $g(w)$ under
		the odds ($q(w)$) parameterization.   The loss functions all result in $\mbox{Pr}(y^*=1\mid x) = F(f^*(x))$.}\label{logistTable} }
  \bigskip
  
  \begin{figure}[h]
 \includegraphics[scale=0.8]{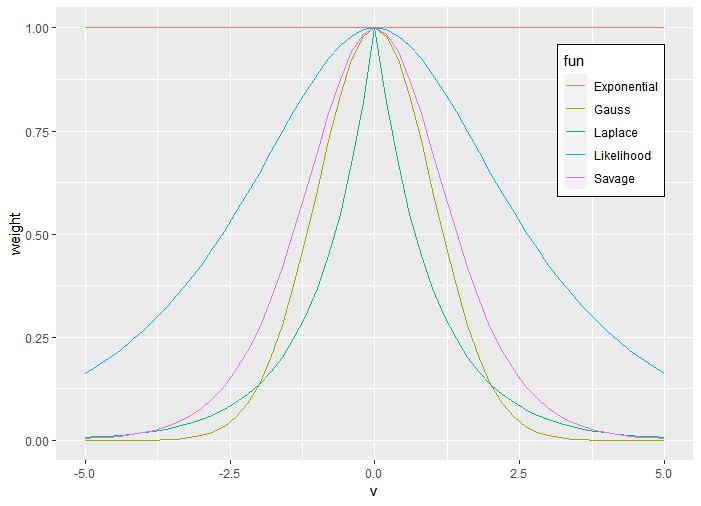}
 \caption{Weight functions $g(v)$ for loss functions in Table \ref{logistTable}. The functions are scaled to have maximum equal to 1. }
 	\label{logistFigure}
\end{figure}
  The following sub-sections explore Gaussian and Laplace loss, as these loss functions appear to be new.
  
  \subsection{Gaussian Loss}
  
  The Gaussian loss function is $\phi_{\mbox{GA}}(v)=1-\Phi(\frac{v+m/2}{\sqrt{m}})$ where $m>0$ and $\Phi(\cdot)$ represents the standard
  normal cumulative distribution function. Gaussian loss is obtained by selecting as a weight function a scaled Gaussian (normal) density; $g(w)=\frac{1}{\sqrt{2\pi m}}e^{-\frac{1}{2m}w^2-m/8}$, see Table \ref{logistTable}. Note that the Gaussian loss function is conformable for the Logistic CDF. In
  the regression context, the Gaussian loss function can be used to obtain consistent estimates of logistic regression coefficients (see Corollary \ref{regressresult}).     The loss function is non-convex, 
  bounded between 0 and 1, and therefore is a smooth approximation to 0-1 loss. 
  
  The shift of $m/2$ occurs automatically from the definition of $g(w)$.  A loss function with a shift of zero  would be outside of $\mathcal{L}_F$, and would result 
  in a degenerate loss function-- we would not be able to consistently estimate a prediction rule, or consistently estimate regression parameters in a regression context.   
  
  An interesting aside:  Setting $m=1$, $\phi_{\mbox{GA}}(v)=1-\Phi(v+1/2)\in\mathcal{L}_F$ while  the Gaussian likelihood loss $\phi(v)= -\log\Phi(v)\in\mathcal{L}_\Phi$. The two loss functions appear very similar, but  conform to different distribution functions.  From a regression parameter estimation perspective, the former loss provides consistent estimators for logistic regression, and the latter for Probit regression. 
  
  The following  provides an example of a loss function closely related to Gaussian loss that has been studied in the context of addressing covariate
  measurement error in logistic regression.
  \begin{example}\label{Buzasexample}   \nocite{Buzas2009} Buzas (2009) considered a modified likelihood score for the logisitc regression model for purposes of deriving an estimating score
  	that remains unbiased in the presence of additive, normal covariate measurement error.  In terms of the class $ \mathcal{\tilde C}_q$, the modified score was defined with weight function
  	\[
  	g(f(x:\beta)) = \frac{(1/m)\Phi^\prime( f(x;\beta)/m)}{\sqrt{ F(f(x;\beta)[1-F(f(x;\beta))]}}
  	\]
  	where $\Phi^\prime (\cdot)$ is the derivative standard normal CDF and $m$ is a scaling constant.  The weight function differs from the Gaussian weight given in Table \ref{logistTable} by the inclusion of $\sqrt{ F(f(x;\beta)[1-F(f(x;\beta))]}$ in the denominator.   By Lemma \ref{classlemma},
  	the derivative of the resulting loss is in $\tilde C_q$ and therefore the loss is in $\mathcal{L}_F$. 
  \end{example}

  \subsection{Laplace Loss}

  The Laplace loss function given in Table \ref{logistTable} is obtained using the kernal of the Laplace
  density function as the weight function $g(w)$.  The resulting loss function, denoted $\phi_L(v)$, is not convex.   The Laplace loss
  function is similar to the Tukey bi-weight function and Huber-type loss functions in that the loss is approximately quadratic for small residuals (and $m\approx 1$), and then tapers off as the magnitude of the residuals increases.   This is seen as follows. 
  
    In terms of standardized logistic residuals, the Laplace loss function ($m>1$) is 
\[
\tilde\phi_L(S)=\phi_L(-\log S^2)=
 \begin{cases}
	|S|^{m+1}, & \text{if $|S|<1$} \\
	1+\frac{m+1}{m-1}(1-\frac{1}{|S|^{m-1}}), & \text{if $|S|\ge 1$}.
\end{cases}
\] 
Similar to Tukey's bi-weight loss, the functional form of the penalty depends on the size of the residual, and the penalty is 
bounded for large $|S|$, see Figure \ref{LaplaceFigure}.

 \begin{figure}[h]
	\includegraphics[scale=0.8]{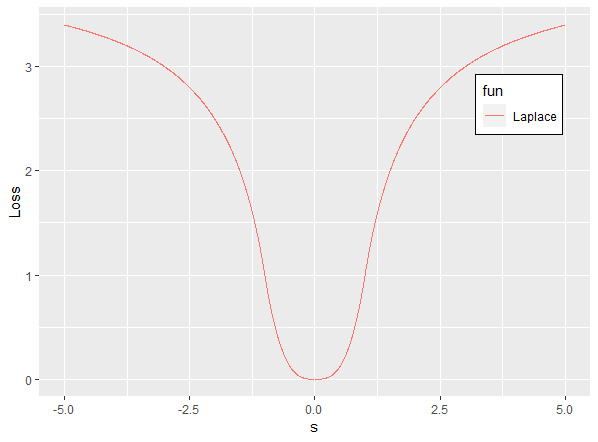}
	\caption{Laplace loss as a function of the standardized logistic regression residual $S$ ($m=2$). }
	\label{LaplaceFigure}
\end{figure}   
Through the choice of $m$, Laplace loss can be made arbitrarily close to 0-1 loss.  It is easily seen that 
\[
\lim_{m\rightarrow\infty}\tilde\phi(S^2)=\begin{cases}
0, & \text{if $|S|<1$} \\
2, & \text{if $|S|\ge 1$}.
\end{cases}
\]
In other words, for large $m$, Laplace loss should closely mimic 0-1 loss.  In a regression context, Laplace loss provides consistent
estimators for the logistic regression model.
The effectiveness of Laplace loss for mitigating the influence of outliers in regression and classification settings
will be studied in future work.
\bigskip

\subsection{Exponential loss}\label{exploss}
 
 Here we explore new interpretations of exponential margin-based loss through the lens of standardized logistic regression residuals.

 In the literature, the standard form for exponential loss is $\phi_E(v)=e^{-v}$ and that is the form used below.  This differs from the form of exponential loss given in Table \ref{logistTable} in that the exponent is not divided by 2. Technically $\phi_E(v)\notin \mathcal{L_F}$ because of the re-scaling in the exponent.  Of course the two forms are equivalent in terms of classification (and easily made so for regression).  
 
 It immediately follows from Theorem \ref{key} that
 $\phi_E(y^*f(x))=S^2$.  Then the empirical risk for exponential loss is equivalent to the (scaled) sum of squared standarized residuals:
 \[
 \frac{1}{n}\sum_{i=1}^n \phi_E(y^*_i f(x_i)) = \frac{1}{n}\sum_{i=1}^n S_i^2
 \] 
 where $S_i^2=S^2(f(x_i))$.
In words, {\it estimation of $f(x)$ through minimizing margin-based exponential loss is equivalent to minimizing the sum of squared standardized logistic regression residuals}.    
 
 \subsubsection{Sensitivity to outliers}
 
 A criticism of exponential loss is that it is sensitive to outliers.  The outlier sensitivity can be explored by examining the behavior of the classifier using the $p$-norm of the standardized logistic residuals. Let  $\mathbf{S}=(S_1,S_2,\dots,S_n)$ and suppose $f(x;\beta)$ is of the form given in \eqref{classFunctions}.  The $p$-norm of $\mathbf{S}$ is 
 
 \begin{equation}\label{pnorm}
 	\left\lVert \mathbf{S}\right\rVert_p =  \left\{\sum_{i=1}^n |S_i|^p \right\}^{\frac{1}{p}} = \left\{\sum_{i=1}^n \exp(-Y^*_if(x_i;\beta) p/2)\right \}^{\frac{1}{p}}.
 \end{equation}
 
 Let $ \hat\beta_p =\arg \min_{\{\beta\}} \left\lVert \mathbf{S}\right\rVert_p^p $ and $\hat f_p(x)=f(x;\hat\beta_p)$.  Estimation with exponential loss is equivalent to choosing $p=2$.
 The following lemma and corollary indicate that the value of $p$ essentially has no impact on estimation or classification. 
 
 \begin{lemma}\label{norm}  The estimates $\hat f_p(x)$ and $\hat f_l(x)$  minimizing the $p$ and $l$ norms of $\mathbf{S}$ have the following relation:
 	
 	\[
 	\hat f_p(x)=\frac{l}{p}\hat f_{l}(x).
 	\]

 \end{lemma}
 \begin{proof}  The lemma follows by noting that, from inspection of  \eqref{pnorm},  $\hat f_p(x) = (1/p)\hat f_1(x)$ and
 	$\hat f_l(x) = (1/l)\hat f_1(x)$.   Combining the two identities gives the result.
 \end{proof}
The lemma implies that the value of the quantity 
 $\{\min_{\beta} \left\lVert \mathbf{S}\right\rVert_p^p \}$
 is independent of the value for $p$.

 The following Corollary establishes that the classifier resulting from minimizing the $p$-norm does not depend on $p$.  The Corollary
 follows directly from Lemma \ref{norm}. 
 \begin{corollary}  
 	Let $\hat C_p(x)=\mbox{sign}\{\hat f_p(x)\}$.  Then $\hat C_p(x)=\hat C_l(x)$ for
 	all $0<p,l<\infty$.
 \end{corollary}
 
 Additional insight from the lemma and corollary can be had by examining the limit as $p$ approaches infinity.  For large $p$,
 \[
 \left\lVert \mathbf{S}\right\rVert_p \approx  \left\lVert \mathbf{S}\right\rVert_\infty \equiv\max_i |S_i|.
 \]
 This approximation provides further insight into the common observation that exponential loss is sensitive to
 outliers.  As shown above, estimation and classification are not dependent on the value of $p$. Then we can consider $p$ large and conclude that exponential loss is, to a close approximation, selecting $f(x;\beta)$ to minimize the maximum of the absolute value of SLRRs, or equivalently to minimize
 the maximum of the magnitude of margins that are negative.

\subsubsection{Exponential risk invariance}

All loss functions $\phi(v)\in\mathcal{L}_F$ have the same optimal value for $f(x)$  at the population level, which will be denoted $f_F^*(x)$.   Note that $f^*_F(x)$ satisfies
$F(f_F^*(x))=p(x)$.

 Denote the conditional risk for exponential loss as $R_{CE}(f(x))=E[\phi_E(y^*f(x))\mid x]=E[e^{-y^*f(x)}\mid x] $. The following Theorem states that the conditional and unconditional risk for exponential loss evaluated at $f_F^*(x)$ is
 invariant to the joint distribution for $(y^*,x)$. 
 This surprising result is perhaps less so when we think about margin-based  exponential loss as the square
 of the standardized logistic regression residuals, and with the understanding that standardized residuals have mean zero and variance one.

 \begin{theorem}  The exponential conditional  risk  evaluated at $f_F^*(x)$ satisfies
 	\[
 	R_{CE}(f_F^*(x)) = 1
 	\]
 	over any joint distribution on $(y^*,x)$. 
 	The conditional (and unconditional) risk is therefore independent of the conditional distribution $p(x)$ and the marginal distribution for $x$.

 \end{theorem} 
 
 \begin{proof} Recall that $F(f_F^*(x))=p(x)$.  Then the conditional risk is given by 
 	\begin{equation*}
 		\begin{split}
 			R_{CE}(f^*_F(x)) & = E[\phi_E(y^*f_F^*(x))\mid x] = E[e^{-y^*f_F^*(x)}\mid x] \\
 			& = E[S^2(f^*_F(x))\mid x]=E\Bigg [\frac{(y-p(x))^2}{p(x)(1-p(x))} \mid x \Bigg ]=1
 		\end{split}
 	\end{equation*}
 	where the last equality follows by noting the expectation is with respect to the underlying conditional distribution $p(x)$, and 
 	that the conditional variance of $y$ is $p(x)(1-p(x))$. 
 	The unconditional risk
 	is $R_E(f^*_F(x))=E[R_{CE}(f^*_F(x))]= E[E[e^{-y^*f^*_F(x)}\mid X]]=1$, regardless of the distribution for $x$. 
 \end{proof}

 The theorem suggests a way of assessing model fit when using any loss function  $\phi(v)\in\mathcal{L}_F$.  Let $\hat f_\phi(x)$ minimize the empirical loss for  $\phi(v)\in\mathcal{L}_F$.   Define the exponential empirical risk as 
 \[
 R_{\mbox{Emp}}(\hat f_\phi(x)) = \frac{1}{n}\sum_{i=1}^n \exp\{-y^*_i\hat f_\phi(x_i)\}.
 \]
 The exponential empirical risk should be  approximately equal to one when we have not over or under fit the data.  Empirical risks greater than 1
 suggest the model is not fit well, and below 1 could suggest over fitting. Additionally, the performance of different loss functions
 in $\mathcal{L}_F$ could be compared via $R_{\mbox{Emp}}$.  AdaBoost uses exponential loss and
 forward stage-wise regression.  The  empirical risk could be evaluated at each iteration, with iterations stopping when $R_{\mbox{Emp}}\approx 1$.  The utility of using $R_{\mbox{Emp}}$ to assess model fit will be explored in future work.

 \subsection{Logistic loss} 
 
  In this sub-section, a new perspective on logistic loss and logistic regression are given in terms of standardized logistic regression residuals.   Recall that logistic loss is defined as $\phi_L(y^*f(x))=\log \left (1+e^{-y^*f(x)}\right )$.

 In Section \ref{exploss}, we showed that minimizing  empirical margin-based exponential loss is equivalent to minimizing the arithmetic mean of squared 
 SLRRs.  Here we show that optimizing margin-based logistic loss (or maximum likelihood estimation if the logistic model is assumed) is equivalent
 to minimizing the geometric mean of (shifted) squared SLRRs.

     The value of $f(x)$ minimizing logistic loss is equivalent to the value of $f(x)$ minimizing the geometric mean of  the values $\{1+S^2_i\}_{i=1}^n$.  That is,
 	\[
 	\arg\min_{f\in \mathcal{F}} \frac{1}{n} \sum_{i=1}^n \phi_L(y_i^*f(x_i)) 
 	= \arg\min_{f\in \mathcal{F}} \left (\prod_{i=1}^n (1+S^2_i)\right )^{1/n}.
 	\] 
 The identity follows readily from Theorem \ref{key} and the definition of the logistic likelihood loss $\phi_L(y^*f(x))$. 	 In terms of regression, the result says that logistic regression maximum likelihood can be thought of as choosing the regression parameters to minimize the geometric mean of shifted, squared SLRRs.  Therefore logistic regression via maximum likelihood is a type of least squares regression.

\subsection{Another perspective on AdaBoost}

The enormously successful boosting algorithm AdaBoost can be understood as forward stagewise additive regression using 
exponential loss, see Friedman et al. (2000)\nocite{Friedman2000}.  Here we argue that AdaBoost  can be understood as forward stagewise 
additive logistic regression, where the objective function is a weighted sum of squared logistic regression residuals, with residuals from prior fits as
the weights.   The development here follows that in Hastie et al. (2009)\nocite{ESL2009}, see pages 343-344.

The basis functions in AdaBoost are weak classifiers $G_m(x)\in\{-1,1\}$. In Friedman et al. (2000) \nocite{Friedman2000}, it is shown that the
$m$th iteration of AdaBoost consists of solving
\begin{equation}\label{adaboost}
	(\theta_m,G_m) = \arg\min_{\theta,G}\sum_{i=1}^n \exp[-y_i^*(f_{m-1}(x_i)+\theta G(x_i))]
\end{equation}  
where $f_{m-1}(x)=\sum_{k=1}^{m-1}\hat\theta_k G_k(x)$ is the sum of the weighted weak classifiers selected in the previous
$m-1$ iterations of the algorithm.  Using the relations between margins and SLRRs established above, it follow that \eqref{adaboost} can 
be written as 
\begin{equation*}\label{adaboost2}
	\begin{split}
	(\theta_m,G_m) & = \arg\min_{\theta,G}\sum_{i=1}^n S^2(f_{m-1}(x_i)+\theta G(x_i)) = \arg\min_{\theta,G}\sum_{i=1}^n  \left [ S^2(f_{m-1}(x_i)) S^2(\theta G(x_i)) \right ]\\
	& =  \arg\min_{\theta,G}\sum_{i=1}^n \left [ \left ( \prod_{k=1}^{m-1}S^2(\hat\theta_kG_k(x_i)) \right ) S^2(\theta G(x_i)) \right ].
	\end{split}      
\end{equation*}  
At each iteration, AdaBoost is minimizing the sum of squared SLRRs, weighted by 
the squared residuals from prior iterations.

\section{Conclusion}

This paper established relations between margin-based loss functions, likelihoods, and logistic regression models. A simple derivation of a characterization of conformable margin-based loss functions was presented. Using the characterization, the
derivative of a large class of margin-based loss functions was shown to be equivalent to weighted likelihood scores.  Additionally, the
margin itself is a function of squared standardized logistic regression residuals, and therefore so too are all
margin-based loss functions.  These relations provide new perspectives on margin-based classification methods, and further 
establish connections between classification via soft or hard classifiers, and regression parameter estimation.  

The simple characterization of margin-based loss functions requires differentiability of the loss, excluding at least
one important loss.  Hinge loss, used in support vector machines, is Bayes consistent, but is not differentiable and therefore not in the class $\mathcal{L}_G$.    Hinge loss can be expressed in terms of squared standardized logistic regression residuals.  However, the effectiveness of hinge loss is probably most easily understood in terms of margins.

\vspace*{0.1in}

\end{document}